\pdfoutput=1

\documentclass[11pt]{article}

\usepackage[]{ACL2023}

\usepackage{times}
\usepackage{latexsym}

\usepackage[T1]{fontenc}

\usepackage[utf8]{inputenc}

\usepackage{microtype}

\usepackage{color}
\usepackage{listings}

\usepackage{amsthm,amssymb}
\usepackage{paralist} 
\usepackage{bbm}
\usepackage{amsmath}
\usepackage{graphicx}
\usepackage{tabu}
\usepackage{placeins}
\usepackage{float}
\usepackage{booktabs}
\usepackage{hyperref}
\usepackage{multirow}
\usepackage{caption}
\newtheorem{theorem}{Theorem}
\usepackage{subcaption}
\usepackage{cleveref}
\usepackage{tikz}
\usepackage{tikzsymbols}
\usepackage{pgfplots}
\usepgfplotslibrary{colormaps}
\usepgfplotslibrary{groupplots}
\pgfplotsset{compat=1.12}
\usetikzlibrary{calc,patterns}

\usepackage{amsmath,amsfonts,bm}

\def\eqref#1{equation~\ref{#1}}

\def\1{\bm{1}}

\def\vzero{{\bm{0}}}

\def\vmu{{\bm{\mu}}}

\def\ve{{\bm{e}}}

\def\vh{{\bm{h}}}

\def\vk{{\bm{k}}}

\def\vq{{\bm{q}}}

\def\vu{{\bm{u}}}
\def\vv{{\bm{v}}}
\def\vw{{\bm{w}}}
\def\vx{{\bm{x}}}
\def\vy{{\bm{y}}}

\def\vmu{{\bm{\mu}}}
\def\v1{{\vec{\bm{\mathbbm{1}}}}}
\def\vzero{{\vec{\bm{0}}}}

\def\mA{{\bm{A}}}

\def\mK{{\bm{K}}}

\def\mM{{\bm{M}}}

\def\mP{{\bm{P}}}
\def\mQ{{\bm{Q}}}

\DeclareMathAlphabet{\mathsfit}{\encodingdefault}{\sfdefault}{m}{sl}
\SetMathAlphabet{\mathsfit}{bold}{\encodingdefault}{\sfdefault}{bx}{n}

\newcommand{\uri}[1]{{\color{blue}Uri:[#1]}}
\newcommand{\ey}[1]{{\color{purple}Eran:[#1]}}
\newcommand{\shaked}[1]{{\color{red}Shaked:[#1]}}

\renewcommand{\shaked}[1]{}
\renewcommand{\uri}[1]{}
\renewcommand{\ey}[1]{}

\title{On the Expressivity Role of LayerNorm in Transformers' Attention}

\author{Shaked Brody$^\dagger$, Uri Alon$^\spadesuit$, Eran Yahav$^\dagger$\\ 
$^\dagger$ Technion, Israel \\   
$^\spadesuit$ Language Technologies Institute, Carnegie Mellon University, USA \\ 
\texttt{\{shakedbr,yahave\}@cs.technion.ac.il}\\
\texttt{ualon@cs.cmu.edu}\\
}

\begin{document}
\maketitle
    \begin{abstract}
        Layer Normalization (LayerNorm) is an inherent component in all Transformer-based models.
        In this paper, we show that LayerNorm is crucial to the expressivity of the multi-head attention layer that follows it. This is in contrast to the common belief that LayerNorm's only role is to normalize the activations during the forward pass, and their gradients during the backward pass.
        
        We consider a geometric interpretation of LayerNorm and show that it consists of two components: 
        \begin{inparaenum}[(a)]
            \item \emph{projection} of the input vectors to a  $d-1$ space that is orthogonal to the $\left[1,1,...,1\right]$ vector,
            and
            \item \emph{scaling} of all vectors to the same norm of $\sqrt{d}$. 
        \end{inparaenum}
        We show that each of these components is \emph{important for the attention layer that follows it in Transformers}:
        \begin{inparaenum}[(a)]
            \item \emph{projection} 
            allows the attention mechanism to create an attention query that attends to all keys equally, offloading the need to learn this operation by the attention;
            and
            \item \emph{scaling} allows each key to potentially receive the highest attention, and prevents keys from being ``un-select-able''.
        \end{inparaenum}
        We show empirically that Transformers do indeed benefit from these properties of LayeNorm in general language modeling and 
        even in computing simple functions such as ``majority''.
        Our code is available at \url{https://github.com/tech-srl/layer_norm_expressivity_role}.
    \end{abstract}

\section{Introduction}\label{Se:Intro}

LayerNorm \cite{ba2016layer} is the most commonly used normalization technique in modern neural networks such as Transformers \cite{vaswani2017attention}.

Originally, \citet{ba2016layer} motivated LayerNorm as an efficient way of normalizing the activations during the forward pass or providing distribution stability as in batch normalization \cite{ioffe2015batch}. Later,
\citet{xu2019understanding} and \citet{xiong2020layer} argued that more importantly than normalizing forward activations, LayerNorm stabilizes the gradients during the \emph{backward} pass.

However, in this work, we show that LayerNorm, which was originally proposed for RNNs, has an additional crucial role in the theoretical and practical \emph{expressivity} of the multi-head attention layer that follows it in Transformers.
\footnote{\citet{xiong2020layer} discuss the differences between placing LayerNorm \emph{before} and \emph{after} a Transformer layer. However, even when placing the LayerNorm \emph{after} the layer, it appears right before the multi-head attention of the \emph{next} layer.} %
That is, LayerNorm makes it easier for the Transformer to learn certain functions during training.

\begin{figure*}[ht!]

    \begin{subfigure}[t]{.45\linewidth}
        \centering
        w/o LayerNorm
        \begin{tikzpicture}[trim axis left,trim axis right, clip, scale=0.75]

    \definecolor{darkslategray38}{RGB}{38,38,38}
    \definecolor{lightgray204}{RGB}{204,204,204}
    \definecolor{color0}{rgb}{0.282352941176471,0.470588235294118,0.815686274509804}
    \definecolor{color1}{rgb}{0.933333333333333,0.52156862745098,0.290196078431373}

    \begin{axis}[
        tick pos=both,
    axis line style={lightgray204},
    x grid style={lightgray204},
    xlabel=\textcolor{darkslategray38}{X},
    xmajorgrids,
    xmin=-1.5, xmax=1.5,
    xtick style={color=darkslategray38},
    xtick={-1.5,-1,-0.5,0,0.5,1,1.5},
    xticklabels={,-1,-0.5,0,0.5,1,},
    y grid style={lightgray204},
    ylabel=\textcolor{darkslategray38}{Y},
    ymajorgrids,
    ymin=-1.5, ymax=1.5,
    ytick style={color=darkslategray38},
    ytick={-1.5,-1,-0.5,0,0.5,1,1.5},
    yticklabels={,-1,-0.5,0,0.5,1,},
    z grid style={lightgray204},
    zlabel style={rotate=-90.0},
    zlabel=\textcolor{darkslategray38}{Z},
    zmajorgrids,
    zmin=-1.5, zmax=1.5,
    ztick style={color=darkslategray38},
    ztick={-1.5,-1,-0.5,0,0.5,1,1.5},
    zticklabels={,-1,-0.5,0,0.5,1,},
    zticklabel pos=right,
    view={60}{10},
    legend cell align={left},
    legend pos=north east,
    legend entries={keys, queries},
    ]
    \addplot3 [color1, opacity=0.5, only marks, mark=*, mark size=2.5, mark options={solid, draw=black}, x=x, y=y, z=z]
    file {figures/keys_no_LN.data};
    \addplot3 [color0, opacity=0.5, only marks, mark=triangle*, mark size=2.5, mark options={solid, draw=black}, x=x, y=y, z=z]
    file {figures/queries_no_LN.data};
    \end{axis}
    
    \end{tikzpicture}
    
        \caption{Without LayerNorm, the model has learned key and query vectors without any apparent geometric structure.}
        \label{fig:no_ln_keys_queries}
    \end{subfigure}
    \hfill
    \begin{subfigure}[t]{.45\linewidth}
        \centering
        w/ LayerNorm
        \begin{tikzpicture}[trim axis left,trim axis right, clip, scale=0.75]

    \definecolor{darkslategray38}{RGB}{38,38,38}
    \definecolor{lightgray204}{RGB}{204,204,204}
    \definecolor{color0}{rgb}{0.282352941176471,0.470588235294118,0.815686274509804}
    \definecolor{color1}{rgb}{0.933333333333333,0.52156862745098,0.290196078431373}
    
    \begin{axis}[
        tick pos=both,
    axis line style={lightgray204},
    x grid style={lightgray204},
    xlabel=\textcolor{darkslategray38}{X},
    xmajorgrids,
    xmin=-1.5, xmax=1.5,
    xtick style={color=darkslategray38},
    xtick={-1.5,-1,-0.5,0,0.5,1,1.5},
    xticklabels={,-1,-0.5,0,0.5,1,},
    y grid style={lightgray204},
    ylabel=\textcolor{darkslategray38}{Y},
    ymajorgrids,
    ymin=-1.5, ymax=1.5,
    ytick style={color=darkslategray38},
    ytick={-1.5,-1,-0.5,0,0.5,1,1.5},
    yticklabels={,-1,-0.5,0,0.5,1,},
    z grid style={lightgray204},
    zlabel style={rotate=-90.0},
    zlabel=\textcolor{darkslategray38}{Z},
    zmajorgrids,
    zmin=-1.5, zmax=1.5,
    ztick style={color=darkslategray38},
    ztick={-1.5,-1,-0.5,0,0.5,1,1.5},
    zticklabels={,-1,-0.5,0,0.5,1,},
    zticklabel pos=right,
    view={60}{10},
    legend cell align={left},
    legend pos=north east,
    legend entries={keys, queries},
    ]
    \addplot3 [color1, opacity=0.5, only marks, mark=*, mark size=2.5, mark options={solid, draw=black}, x=x, y=y, z=z]
    file {figures/keys_with_LN.data};
    \addplot3 [color0, opacity=0.5, only marks, mark=triangle*, mark size=2.5, mark options={solid, draw=black}, x=x, y=y, z=z]
    file {figures/queries_with_LN.data};
    \end{axis}
    
    \end{tikzpicture}
    
        \caption{LayerNorm projects the key vectors onto the same hyperplane so that the model can
        learn to align the queries to be orthogonal to the keys.}
        \label{fig:ln_keys_queries}
    \end{subfigure}

    \begin{subfigure}[t]{.45\linewidth}
        \centering
        \input{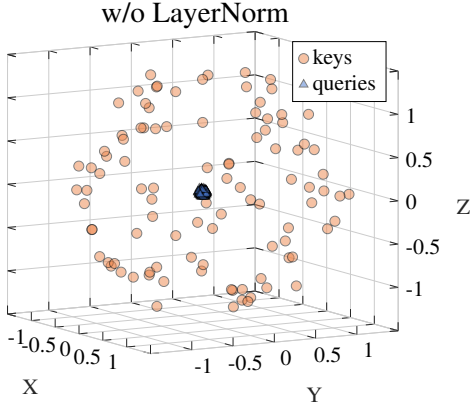}
        \caption{Without LayerNorm, there are ``unselectable'' key vectors that cannot be selected by getting the maximal attention score (marked in darker colors).}
        \label{fig:no_ln_heat_map}
    \end{subfigure}
    \hfill
    \begin{subfigure}[t]{.45\linewidth}
        \centering
        \input{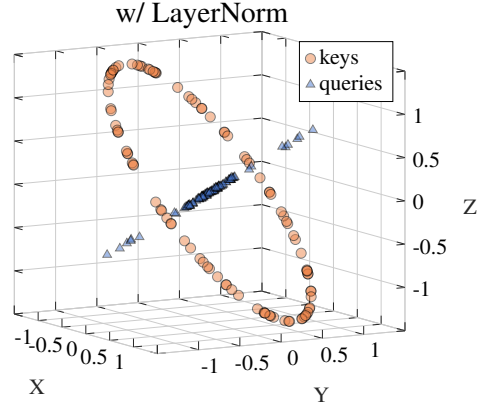}
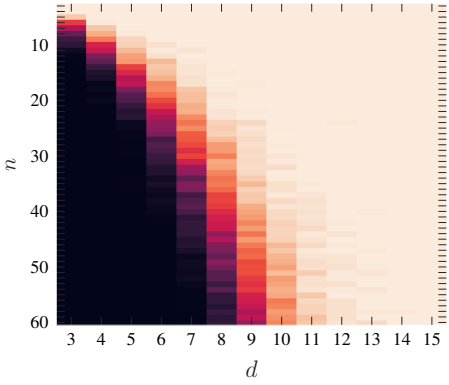
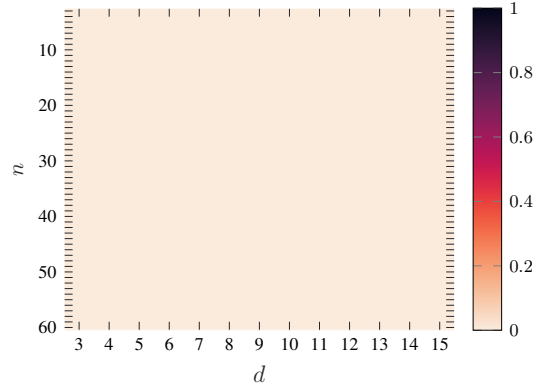
        \caption{LayerNorm eliminates the problem of ``unselectable'' key vectors: Applying LayerNorm allows any key to get the highest attention score.}
        \label{fig:ln_heat_map}
    \end{subfigure}

    \caption{
    \Cref{fig:no_ln_keys_queries,fig:ln_keys_queries} show the effect of \emph{projection} in LayerNorm, which makes all 
    keys lie on the hyperplane that is orthogonal to the $\v1$ vector.
    \Cref{fig:no_ln_heat_map,fig:ln_heat_map} show the effect of \emph{scaling}, where $n$ is the number of vectors, $d$ is the dimension, and the color represents the average fraction of ``unselectable'' key vectors.
    } 
    \label{fig:first}
\end{figure*}

LayerNorm can be seen as two independent 
components: \emph{projection} and \emph{scaling}, that were merged into a single operator.
First, LayerNorm \emph{projects} its inputs onto a particular $d-1$ space that is orthogonal to the ``ones'' $\v1=\left[1,1,...,1\right]$  vector. This allows the attention layer that follows the LayerNorm to create queries that are close to $\v1$, and thus attend to all keys equally, when needed, regardless of the identity of the keys. 
In \Cref{Se:Importance} we show that this projection helps, for example, computing the ``majority'' among token types in a sequence.

\Cref{fig:no_ln_keys_queries} shows how without LayerNorm, 
the keys and queries in a Transformer's attention have no apparent geometric structure.
In contrast, \Cref{fig:ln_keys_queries} shows that LayerNorm has projected all keys to the hyperplane that is orthogonal to the $\v1$ vector. Further, the attention mechanism has learned queries that are close to $\v1$, making them attend equally to any possible key, when trained to compute ``majority''.
We analyze and prove this in \Cref{Se:Average}.

The second component of LayerNorm is \emph{scaling}: We show that LayerNorm scales the projected input to have an $\ell^2$ norm of exactly $\sqrt{d}$.
In \Cref{Se:Importance}, we show that scaling the input vectors prevents the problem of ``unselectable'' keys \cite{demeter2020stolen, grivas2022low}, where some key vectors are contained in the convex hull formed by the other keys, 
and thus can never get the highest attention score.
\Cref{fig:no_ln_heat_map,fig:ln_heat_map} show
the average fraction (out of 100 runs) of ``unselectable'' vectors which were randomly drawn from the normal distribution. 
As shown in \Cref{fig:no_ln_heat_map}, without LayerNorm, the probability of getting ``unselectable'' keys can be very high in certain settings.
Nonetheless, as shown in \Cref{fig:ln_heat_map}, with LayerNorm, every key vector is always selectable. We analyze and prove this in \Cref{Se:Unbounded}.

These results reveal new aspects of the commonly used LayerNorm and show its importance to the attention mechanism in Transformers.

\section{Decomposing LayerNorm}\label{Se:Decomposing}

Given an input $\vx\in\mathbb{R}^{d}$, LayerNorm is defined as the following quotient:\footnote{Following \citet{xu2019understanding} and for simplicity, we drop the learned bias and gain terms.}
\begin{equation}\label{eq:LayerNorm}
    \vy= \frac{\vx - \vmu}{\sigma}
\end{equation}
where $\mu$ is the coordinate-wise average of $\vx$ and $\sigma$ is the coordinate-wise standard deviation:
\begin{equation}
    \mu = \frac{1}{d}\sum_{i=1}^d x_i, \quad \sigma = \sqrt{\frac{1}{d}\sum_{i=1}^d (x_i - \mu)^2}
\end{equation}
\begin{equation}
    \vmu= \left[\mu,\mu,...,\mu\right]\in\mathbb{R}^d
\end{equation}

We start with the numerator $\vx-\vmu$ and show that it corresponds to the projection of $\vx$ onto the hyperplane $\mathcal{H}$ defined by the normal vector $\v1=\left[1,1,...,1\right]\in\mathbb{R}^d$:
\uri{I changed this: }
\begin{equation}
    \begin{split}
        &\left(\vx - \vmu\right) \cdot \v1 = \vx \cdot \v1 - \vmu \cdot \v1 \\
        &\sum_{i=1}^d x_i - \left(\frac{1}{d}\sum_{i=1}^d x_i\right) \cdot d  = 0 
    \end{split}
\end{equation}
That is, $\vx - \vmu$ is always orthogonal to the $\v1$ vector.
Next, we show that the denominator scales the projected vector to have a norm of exactly $\sqrt{d}$:
\begin{equation}
    \begin{split}
    ||\vx-&\vmu|| = \sqrt{\sum_{i=1}^d \left(x_i-\mu\right)^2}\\
    &= \sqrt{d}\sqrt{\frac{1}{d}\sum_{i=1}^d \left(x_i-\mu\right)^2} = \sqrt{d} \cdot \sigma
    \end{split}
\end{equation}
Then, 
 $\sigma$ in the denominator of LayerNorm scales the projected vector to have a norm of exactly $\sqrt{d}$.

LayerNorm can thus be seen as two independent components:
\begin{inparaenum}[(a)]
    \item \emph{projection} of the input vectors onto the hyperplane orthogonal to $\v1$, and 
    \item \emph{scaling} of the projected vectors to have a norm of $\sqrt{d}$.
\end{inparaenum}

\section{Expressivity Role in Attention}\label{Se:Importance}
Each of the components of LayerNorm supports the  Transformer's attention in a different way.
\emph{Projection} helps 
creating a query that attends to all keys equally, when needed,
while \emph{scaling} helps the model to avoid the problem of ``unselectable'' keys.

Recall that in Transformers, given vectors $\vq$ and $\vk$, the attention scoring function 
is defined as:
\begin{equation}
    s\left(\vq,\vk\right) = \frac{\left(\vq\mQ\right)\left(\vk\mK\right)^\top}{\sqrt{d}}=
    \left(\frac{\vq\mQ\mK^\top}{\sqrt{d}}\right)\vk^\top
\end{equation}
From this point, we refer to $\left(\frac{\vq\mQ\mK^\top}{\sqrt{d}}\right)$ as ``query'', and to $\vk$ as ``key''.

\subsection{Projection}

Projecting all attention keys to the same hyperplane
can help attention
attending to all keys equally.
Since all projected keys
are orthogonal to the hyperplane's normal $\v1$,
this fact can be exploited in the training process by learning weights such that
the queries 
will be parallel to $\v1$. That is, the attention can learn weights such that
$\left(\frac{\vq\mQ\mK^\top}{\sqrt{d}}\right) \approx c\cdot \v1$, which will result in
$\left(\frac{\vq\mQ\mK^\top}{\sqrt{d}}\right) \cdot \vk \approx 0$ and thus 
$s\left(\vq,\vk\right)=0$
for any key $\vk$.

Giving an equal score to all the keys can help, for example, in computing ``majority'', where the model needs to find the most frequent token in the input.
In \Cref{Se:Average}, we show that in the ``majority'' task, a Transformer learns
to align the queries to be orthogonal to the keys, 
allowing a much faster convergence. %

\subsection{Scaling}
\emph{Scaling} the attention keys to the same size allows a Transformer to avoid the problem of ``unselectable'' keys, where there are keys to which the attention cannot focus and give the highest score.

Let $\mathcal{S}=\left\{\vh_1,...,\vh_{n-1},\vh_n\right\}$ be a set of key vectors, such that $\vh_n$ lies within the convex hull formed by the other vectors in $\mathcal{S}$.
Due to the linearity of the attention scoring function $s$, the attention mechanism cannot select $\vh_n$ by giving it the highest attention score. We formulate this in the following theorem:

\begin{theorem}
    Given a set of vectors $\mathcal{S}=\left\{\vh_1,...,\vh_{n-1},\vh_n\right\}$ such that $\vh_n$ is interior to the convex hull of $\mathcal{S}$, then for all $\vv\in\mathbb{R}^d$ (s.t. $\vv \neq \vzero$):
\begin{equation*}
    \max_{i\in[n-1]}
        \vv^\top\vh_i
     > \vv^\top\vh_n
     \label{Th:theorem-1}
    \end{equation*}

\end{theorem}
This means that key vectors that are inside the convex hull cannot be selected by getting the highest attention score.
Applying LayerNorm to the keys ensures that all keys are \emph{scaled} to the same size, and thus none of them lies inside the convex hull of $\mathcal{S}$. 
This allows the attention mechanism to potentially focus and select any desired key. 
The proof of \Cref{Th:theorem-1} is provided in \Cref{Ap:proof-theorem-1}.

In \Cref{Se:Unbounded} we show that this happens in practice when we train a Transformer on language modeling: There are key vectors that lie within the convex hull and therefore cannot be selected by receiving the maximal attention.

\section{Experimental Results}\label{Se:Results}
In this section, we empirically show the effects of the LayerNorm components -- \emph{projection} and \emph{scaling} -- on the attention mechanism in Transformers.
We first show how a Transformer learns to use the \emph{projection} of LayerNorm to compute ``majority'';
then, we show that \emph{scaling}
allows the model to avoid the problem of ``unselectable'' keys, allowing the attention mechanism to focus on any key.

\subsection{Computing Majority}\label{Se:Average}
We demonstrate the ability to compute ``majority''  using the \emph{projection} property. In this task, the goal is to predict the majority token type in a sequence. Given a sequence of tokens $t_1,t_2,...,t_n\in \left\{C_1,C_2,...,C_k\right\}$, the goal is to predict the token type $C_i$ that occurs the most among $t_1,t_2,...,t_n$. For $a,a,b,b,b,c,c$, for example, the model is trained to predict the output $b,b,b,b,b,b,b$.
Solving this task can be performed simply using exact averaging of the keys.

We trained a single-layer Transformer encoder with dimension $d=8$ and a single attention head. We experimented with standard LayerNorm compared to a LayerNorm without \emph{projection} (having a numerator of simply $\vx$ in \Cref{eq:LayerNorm}, similarly to the LayerNorm variant of \citet{zhang2019root}). We trained each model 10 times using different random seeds.

\begin{figure}[t!]
        \centering
        \begin{subfigure}[t]{0.95\linewidth}
            \input{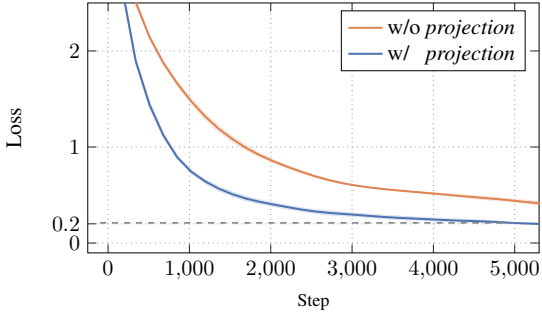}
            \caption{
            Training loss:
            \emph{With} projection, the model converges faster compared to the model \emph{without} projection, which required 3x more steps. 
            }
            \label{fig:majority_loss}
        \end{subfigure}
        \\[1ex]
        \begin{subfigure}[t]{0.95\linewidth}
            \input{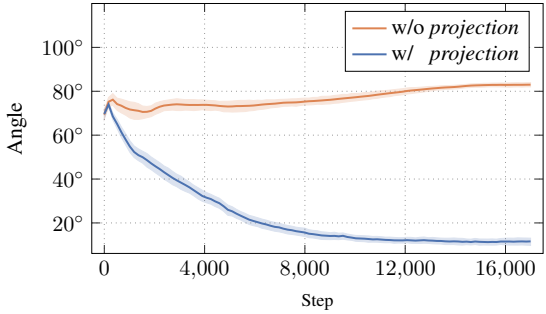}
            \caption{
            The mean angle of the queries to the $\v1$ vector.
            \emph{With}  projection, 
            since all keys are orthogonal to $\v1$,
            the model has learned to align the queries to be parallel to $\v1$ and thus give equal attention to all keys.
            }
            \label{fig:majority_angle}
        \end{subfigure}
        \caption{
        The training loss and mean angle of the queries to $\v1$ in the ``majority'' task across 10 runs with and without \emph{projection}.
        }
        \vspace{-4mm}
        \label{fig:majority}
\end{figure}

\paragraph{Results}
\Cref{fig:majority_loss} shows that \emph{with} projection, the model converges faster compared to \emph{without} projection. We hypothesize that since
all key vectors are orthogonal to the $\v1$ vector, the model can exploit this geometric structure, and learning the task is made easier. 
\Cref{fig:majority_angle} shows that indeed, the model \emph{with} projection has learned to align the queries to the $\v1$ vector, decreasing the angle between the queries and $\v1$.
On the contrary, a model \emph{without} projection has to learn 
 to solve this task ``from scratch''. This model also converged eventually, but it required 3x more training steps.
\Cref{fig:majority-test} shows a similar trend in the models' test accuracies.

\input{majority_test}

\subsection{Unselectable Keys}\label{Se:Unbounded}

\begin{table}[t!]
    \centering
        \begin{tabu}{lccccc}
        \toprule
        Model & $L_1$ & $L_2$ & $L_3$ & $L_4$ \\
        \midrule
        w/o \emph{scaling} & 51.0 & 32.2 & 34.7 & 36.8 \\
        w/\hphantom{o}  \emph{scaling} & \textbf{0} & \textbf{0} & \textbf{0} & \textbf{0} \\
    \bottomrule
    \end{tabu}
    \caption{The fraction of ``unselectable'' key vectors right before the attention mechanism of each layer of a language model without \emph{scaling}. LayerNorm solves the ``unselectable'' keys problem using the \emph{scaling} property. Without \emph{scaling}, there are key vectors that cannot be selected by the attention mechanism.}
    \label{tab:unargmaxable-no-scale}
\end{table}

We examined the fraction of ``unselectable'' keys in a Transformer model with and without the \emph{scaling} component of LayerNorm using the method presented in \citet{grivas2022low}.
We trained a 4-layer language model (based on GPT2 architecture \cite{radford2019language}) with $d=8$ on Wikipedia %
\footnote{\url{https://huggingface.co/datasets/wikipedia}, \texttt{20220301.en} split.} 
for 50K steps, and analyzed the inputs to the attention layer in each layer using sequences from the validation set of SQuAD \cite{squad}.

\paragraph{Results}
\Cref{tab:unargmaxable-no-scale} shows the following results: \emph{Without} \emph{scaling}, there are at least 32\% ``unselectable'' keys that cannot receive maximal attention score \emph{in each layer}. In contrast, \emph{scaling} removes this problem and allows the model to potentially focus on any key. 
In \Cref{fig:unselectable_train_loss,fig:unselectable_test_loss}, we show that this difference in ``unselectable'' keys is also reflected in higher training and test losses for the model that does not use \emph{scaling}.

\begin{figure}[ht!]
    \centering
    \begin{subfigure}[t]{0.95\linewidth}
            \begin{tikzpicture}[scale=0.75]
    \definecolor{darkslategray38}{RGB}{38,38,38}
    \definecolor{lightgray204}{RGB}{204,204,204}
    \definecolor{steelblue76114176}{RGB}{76,114,176}
    \definecolor{peru22113282}{RGB}{221,132,82}

    \begin{axis}[
        axis line style={black},
        legend cell align={left},
        legend style={
        },
        height=6cm,
        width=9.5cm,
        legend pos=north east,
        legend entries={w/o \emph{scaling},w/\hphantom{o} \emph{scaling}},
        tick align=inside,
        tick pos=both,
        scaled x ticks = false,
        grid style={dotted, gray},
        xlabel={Step},
        xlabel style={font=\small},
        xmajorgrids,
        xmin=5000, xmax=52475,
        xtick style={color=white!15!black},
        ylabel={Loss},
        ylabel style={font=\small},
        ymajorgrids,
        ymin=7.10745, ymax=8,
        ytick style={color=white!15!black},
        ylabel near ticks,
        y label style={at={(-0.130,0.5)}},
        extra y ticks={0.2},
    ]
    \addplot [line width=1pt, peru22113282]
    table {%
    500 10.7058
    1000 10.4713
    1500 10.233
    2000 9.988
    2500 9.7409
    3000 9.4995
    3500 9.2659
    4000 9.0491
    4500 8.8372
    5000 8.6514
    5500 8.4808
    6000 8.3204
    6500 8.1783
    7000 8.0756
    7500 7.9734
    8000 7.8918
    8500 7.8419
    9000 7.7916
    9500 7.7735
    10000 7.7516
    10500 7.7395
    11000 7.7354
    11500 7.7229
    12000 7.704
    12500 7.6996
    13000 7.6818
    13500 7.6636
    14000 7.6525
    14500 7.6425
    15000 7.6227
    15500 7.6088
    16000 7.601
    16500 7.6006
    17000 7.5847
    17500 7.5752
    18000 7.5768
    18500 7.5696
    19000 7.5692
    19500 7.5531
    20000 7.5493
    20500 7.5386
    21000 7.5386
    21500 7.5347
    22000 7.5243
    22500 7.5186
    23000 7.5111
    23500 7.5075
    24000 7.4996
    24500 7.4934
    25000 7.4855
    25500 7.4803
    26000 7.4798
    26500 7.4743
    27000 7.4663
    27500 7.4487
    28000 7.4531
    28500 7.4436
    29000 7.4346
    29500 7.4389
    30000 7.4296
    30500 7.4174
    31000 7.4125
    31500 7.4153
    32000 7.4075
    32500 7.4025
    33000 7.3991
    33500 7.3904
    34000 7.3926
    34500 7.3725
    35000 7.3796
    35500 7.3886
    36000 7.3783
    36500 7.3753
    37000 7.3631
    37500 7.3693
    38000 7.366
    38500 7.3654
    39000 7.3516
    39500 7.3564
    40000 7.3519
    40500 7.3519
    41000 7.3385
    41500 7.3419
    42000 7.3382
    42500 7.3345
    43000 7.3299
    43500 7.3429
    44000 7.3371
    44500 7.3361
    45000 7.3381
    45500 7.3346
    46000 7.3288
    46500 7.3313
    47000 7.3254
    47500 7.325
    48000 7.3257
    48500 7.3209
    49000 7.3301
    49500 7.3272
    50000 7.3245
    };

    \addplot [line width=1pt, steelblue76114176]
    table {%
    500 10.7054
    1000 10.4719
    1500 10.234
    2000 9.9894
    2500 9.742
    3000 9.5006
    3500 9.2671
    4000 9.0504
    4500 8.8386
    5000 8.6529
    5500 8.4822
    6000 8.3217
    6500 8.1792
    7000 8.0763
    7500 7.9743
    8000 7.8927
    8500 7.8425
    9000 7.792
    9500 7.7736
    10000 7.7513
    10500 7.7387
    11000 7.7339
    11500 7.7206
    12000 7.7009
    12500 7.6961
    13000 7.6777
    13500 7.6595
    14000 7.6483
    14500 7.6379
    15000 7.6176
    15500 7.6023
    16000 7.5934
    16500 7.591
    17000 7.5729
    17500 7.5614
    18000 7.5602
    18500 7.5507
    19000 7.5476
    19500 7.5283
    20000 7.5217
    20500 7.5083
    21000 7.505
    21500 7.4978
    22000 7.4848
    22500 7.4762
    23000 7.466
    23500 7.4606
    24000 7.4511
    24500 7.4439
    25000 7.435
    25500 7.4296
    26000 7.4287
    26500 7.4228
    27000 7.4146
    27500 7.3979
    28000 7.4028
    28500 7.394
    29000 7.3855
    29500 7.3903
    30000 7.3813
    30500 7.37
    31000 7.3648
    31500 7.3683
    32000 7.3614
    32500 7.3566
    33000 7.3536
    33500 7.345
    34000 7.3479
    34500 7.3275
    35000 7.3352
    35500 7.3445
    36000 7.334
    36500 7.3308
    37000 7.319
    37500 7.3254
    38000 7.3223
    38500 7.3218
    39000 7.3079
    39500 7.3122
    40000 7.3082
    40500 7.3082
    41000 7.2951
    41500 7.2986
    42000 7.2947
    42500 7.2915
    43000 7.2873
    43500 7.3004
    44000 7.294
    44500 7.2936
    45000 7.2964
    45500 7.292
    46000 7.2866
    46500 7.2892
    47000 7.2829
    47500 7.2825
    48000 7.2831
    48500 7.2788
    49000 7.2877
    49500 7.2855
    50000 7.2829
    };

\end{axis}

\end{tikzpicture}
        \caption{Training loss: \emph{With} scaling, the model converge faster compared to the model \emph{without} scaling.}
        \label{fig:unselectable_train_loss}
    \end{subfigure}
    \begin{subfigure}[t]{0.95\linewidth}
            \begin{tikzpicture}[scale=0.75]
    \definecolor{darkslategray38}{RGB}{38,38,38}
    \definecolor{lightgray204}{RGB}{204,204,204}
    \definecolor{steelblue76114176}{RGB}{76,114,176}
    \definecolor{peru22113282}{RGB}{221,132,82}

    \begin{axis}[
        axis line style={black},
        legend cell align={left},
        legend style={
        },
        height=6cm,
        width=9.5cm,
        legend pos=north east,
        legend entries={w/o \emph{scaling},w/\hphantom{o} \emph{scaling}},
        tick align=inside,
        tick pos=both,
        scaled x ticks = false,
        grid style={dotted, gray},
        xlabel={Step},
        xlabel style={font=\small},
        xmajorgrids,
        xmin=5000, xmax=52475,
        xtick style={color=white!15!black},
        ylabel={Loss},
        ylabel style={font=\small},
        ymajorgrids,
        ymin=7.10745, ymax=8,
        ytick style={color=white!15!black},
        ylabel near ticks,
        y label style={at={(-0.130,0.5)}},
        extra y ticks={0.2},
    ]
    \addplot [line width=1pt, peru22113282]
    table {%
    500 10.58815765
    1000 10.35688877
    1500 10.11849117
    2000 9.87612915
    2500 9.63599968
    3000 9.402089119
    3500 9.178215027
    4000 8.967827797
    4500 8.772199631
    5000 8.594137192
    5500 8.433856964
    6000 8.291932106
    6500 8.167560577
    7000 8.064609528
    7500 7.978945732
    8000 7.913676739
    8500 7.866682529
    9000 7.835698128
    9500 7.817951202
    10000 7.804569244
    10500 7.794483185
    11000 7.783994675
    11500 7.770471573
    12000 7.755699635
    12500 7.740699291
    13000 7.723731518
    13500 7.70853138
    14000 7.694862843
    14500 7.682023048
    15000 7.67058897
    15500 7.660918713
    16000 7.652450085
    16500 7.645085812
    17000 7.638020992
    17500 7.631177902
    18000 7.625000477
    18500 7.618332863
    19000 7.61164856
    19500 7.606098175
    20000 7.600143433
    20500 7.59510994
    21000 7.589673519
    21500 7.583784103
    22000 7.579368114
    22500 7.573656559
    23000 7.569850445
    23500 7.562768936
    24000 7.558451176
    24500 7.553297997
    25000 7.548194885
    25500 7.542777061
    26000 7.536319733
    26500 7.531676769
    27000 7.525649548
    27500 7.517554283
    28000 7.512712002
    28500 7.505766869
    29000 7.49943161
    29500 7.494571686
    30000 7.48803091
    30500 7.482662201
    31000 7.47705555
    31500 7.471026421
    32000 7.465358257
    32500 7.460614204
    33000 7.455510139
    33500 7.451537609
    34000 7.446690559
    34500 7.442859173
    35000 7.438271523
    35500 7.434169769
    36000 7.430821419
    36500 7.426640511
    37000 7.423700333
    37500 7.42009449
    38000 7.416528225
    38500 7.414104462
    39000 7.411118031
    39500 7.408524513
    40000 7.406078815
    40500 7.404069901
    41000 7.401488781
    41500 7.398902416
    42000 7.397368431
    42500 7.394959927
    43000 7.393601418
    43500 7.391349792
    44000 7.390385628
    44500 7.389098167
    45000 7.387676716
    45500 7.386381626
    46000 7.38533783
    46500 7.384510517
    47000 7.383586884
    47500 7.382931709
    48000 7.382424831
    48500 7.381913185
    49000 7.381599426
    49500 7.381298542
    50000 7.381257534
    };
    \addplot [line width=1pt, steelblue76114176]
    table {%
    500 10.58867645
    1000 10.3576088
    1500 10.11983585
    2000 9.87726593
    2500 9.63749218
    3000 9.403622627
    3500 9.179793358
    4000 8.969419479
    4500 8.773835182
    5000 8.595761299
    5500 8.435460091
    6000 8.292439461
    6500 8.16880703
    7000 8.064229012
    7500 7.980135918
    8000 7.914477825
    8500 7.866100788
    9000 7.8355937
    9500 7.81762886
    10000 7.803906918
    10500 7.793042183
    11000 7.781525135
    11500 7.766689301
    12000 7.750727177
    12500 7.734770298
    13000 7.717261314
    13500 7.701801777
    14000 7.687759876
    14500 7.674412251
    15000 7.662183285
    15500 7.651196957
    16000 7.640724182
    16500 7.630444527
    17000 7.620294094
    17500 7.609977722
    18000 7.599876881
    18500 7.590127945
    19000 7.579855919
    19500 7.569705963
    20000 7.559968948
    20500 7.550923824
    21000 7.541402817
    21500 7.532032013
    22000 7.523539543
    22500 7.51459074
    23000 7.507865906
    23500 7.498946667
    24000 7.492763042
    24500 7.486328125
    25000 7.48077631
    25500 7.474948883
    26000 7.468165874
    26500 7.463563442
    27000 7.458331108
    27500 7.451378822
    28000 7.446918488
    28500 7.441514492
    29000 7.436664104
    29500 7.432415962
    30000 7.426876068
    30500 7.422685623
    31000 7.41805172
    31500 7.412661076
    32000 7.407913685
    32500 7.403532982
    33000 7.399259567
    33500 7.395570755
    34000 7.391477585
    34500 7.387401104
    35000 7.383240223
    35500 7.379859447
    36000 7.376809597
    36500 7.372546673
    37000 7.36975193
    37500 7.366511345
    38000 7.363131046
    38500 7.361083984
    39000 7.358507156
    39500 7.35581255
    40000 7.35347271
    40500 7.351573944
    41000 7.349138737
    41500 7.347048759
    42000 7.345486164
    42500 7.343604565
    43000 7.342066288
    43500 7.340027332
    44000 7.339292049
    44500 7.338184357
    45000 7.337085724
    45500 7.335737705
    46000 7.334742069
    46500 7.33415699
    47000 7.333363533
    47500 7.332689762
    48000 7.332261562
    48500 7.331927776
    49000 7.331605434
    49500 7.331315041
    50000 7.331252098
    };
\end{axis}

\end{tikzpicture}
        \caption{Test loss: \emph{With} scaling, the model achieves lower test loss faster compared to the model \emph{without} scaling.}
        \label{fig:unselectable_test_loss}
    \end{subfigure}
    \caption{The training and test loss in the language modeling task.}
    \label{fig:unselectable}
\end{figure}
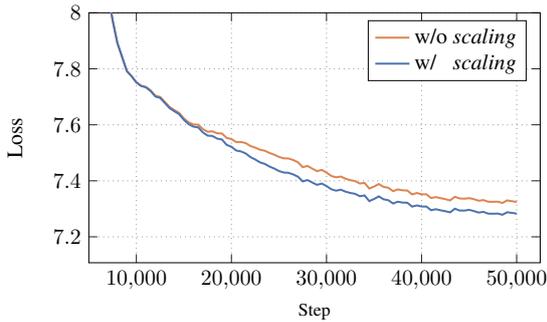
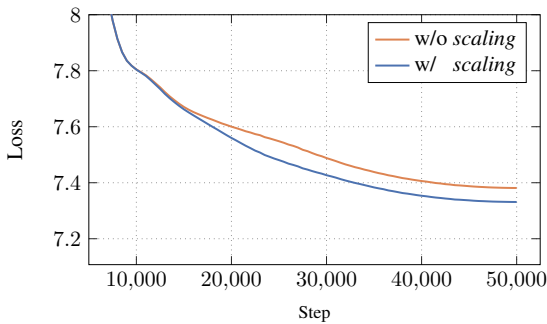

\section{Conclusion}\label{Se:Conclusion}
In this paper, we show that the commonly used LayerNorm component is crucial not only for the optimization process but also for the expressivity of attention in Transformers.
We decompose LayerNorm into two geometric operations: \emph{projecting} the input vectors onto a subspace that is orthogonal to the $\v1$ vector, and \emph{scaling} the projected vectors to have the same norm $\sqrt{d}$. 

We show that \emph{projection} helps to compute even simple tasks such as ``majority'' by performing an exact average of the keys and that \emph{scaling} helps to avoid the problem of ``unselectable'' keys. 

These findings are important for the community's understanding of attention and expressivity in Transformers. Further, these results raise a variety of follow-up questions, such as: why should the keys be orthogonal to the $\v1$ vector, instead of any, learnable, different vector for every layer? And what would happen if we force each layer's keys to be orthogonal to \emph{multiple} normal vectors?
To this end, we make our code publicly available at \url{https://github.com/tech-srl/layer_norm_expressivity_role}.

    \section{Limitations}

    In this work, we found that the implications of the geometric properties of LayerNorm affect mainly small models and are less evident for larger models. We hypothesize that with a large hidden dimension, a Transformer model can find other solutions for computing ``majority`` using gradient descent and is, therefore, less dependent on the \emph{projection} component. Further, we believe that the \emph{scaling} component is less useful for high dimensional models, 
    since with higher dimensions, it is less likely to encounter a set of vectors where some of them lie within the convex hull of the others. 
    Therefore, we encourage the community to use LayerNorm before attention layers, especially for small models that operate on long sequences.

    Moreover, the \emph{projection} component is clearly a linear operator that can be expressed by a linear layer before the LayerNorm, as we show in \Cref{Ap:projection}. Nevertheless, the importance of the projection holds as we discuss in \Cref{Se:Importance}, and the benefit of using this operator explicitly in LayerNorm is shown in \Cref{Se:Average}.  
\section*{Acknowledgements}
We thank Gail Weiss for the helpful discussions.

\bibliography{bib}{}
\bibliographystyle{acl_natbib}

\appendix
\FloatBarrier
\section{Proof of \Cref{Th:theorem-1}}\label{Ap:proof-theorem-1}
\addtocounter{theorem}{-1}

\begin{theorem}
	Given a set of vectors $\mathcal{S}=\left\{\vh_1,...,\vh_{n-1},\vh_n\right\}$ such that $\vh_n$ is interior to the convex hull of $\mathcal{S}$, then for all $\vv\in\mathbb{R}^d$ (s.t. $\vv \neq \vzero$):
\begin{equation*}
    \max_{i\in[n-1]}
        \vv^\top\vh_i
     > \vv^\top\vh_n
     \label{Th:theorem-1}
    \end{equation*}

\end{theorem}

We prove \Cref{Th:theorem-1} using \Cref{Th:stollen}, presented and proved in \citet{demeter2020stolen}:
\begin{theorem}[{\citet{demeter2020stolen}}]
    Let $C$ be the convex hull of the embeddings $\left\{x_i\right\}$ of a vocabulary $V$. If an embedding $x_i$ for word $w_i\in V$ is interior to $C$, then the maximum probability $P\left(w_i\right)$ assigned to $w_i$ using a dot-product softmax is bounded by the probability assigned to at least one word $w_i$ whose embedding is on the convex hull.
    \label{Th:stollen}
\end{theorem}

\begin{proof}[Proof of \Cref{Th:theorem-1}]
    According to \Cref{Th:stollen}, since $\vh_n$ is interior to the convex hull of $\mathcal{S}$, 
    the maximum probability assigned to $\vh_n$
    is bounded by the maximum probability assigned to $\vh_i$ for some $i \in [1-n]$, that is on the convex hull of $\mathcal{S}$.
    Since probability in Transformers is computed as a dot-product of the final hidden state of the Transformer $\vu$ and the embedding vector, we can write:
    \begin{equation}
        \vu^{\top}\vh_i > \vu^{\top}\vh_n
    \end{equation}
    for \emph{any} $\vu \in \mathbb{R}^d$ (the probability is computed as a \emph{softmax} of the dot-product logits, but since softmax is a monotonic function, a higher probability after the softmax necessarily implies a higher logit score).

    Therefore, it also implies that for any $\vv = \vu \in \mathbb{R}^d$:
    \begin{equation}
        \max_{i\in[n-1]}
        \vv^\top\vh_i
     > \vv^\top\vh_n
    \end{equation}

\end{proof}

\section{Characteristics of the Normalized Vectors}\label{Se:character}
In this section, we discuss the characteristics of the LayerNorm inputs that are been normalized to the same point. 
Recall that the \emph{projection} ensures that the normalized output lies on the hyperplane $\mathcal{H}$ defined by the normal vector $\v1=\left[1,1,...,1\right]\in\mathbb{R}^d$.

Let $\vv$ be a unit vector in $\mathcal{H}$:
\begin{equation}
    \vv\bot\v1 \land ||\vv||=1
\end{equation}
Therefore
\begin{equation}
    \sum_{i=1}^d \vv_i = 0
\end{equation}
Let $\mathcal{M}$ be a \textbf{2D plane} that is defined using $\vv$ and $\v1$. Its parametric representation is:
\begin{equation}
    \mathcal{M}: s \vv + t \v1
\end{equation}
Finally, let $\vx$ be a vector in $\mathcal{M}$. Therefore, there exist $\alpha,\beta\in\mathbb{R}$ 
such that
\begin{equation}
   \vx =\alpha \vv + \beta \v1
\end{equation}
Next, we apply LayerNorm to $\vx$. First we project $\vx$ onto $\mathcal{H}$:
\begin{equation}
    \begin{split}
        &\vx - \vmu = \\
        &=\alpha\vv + \beta\v1 - \frac{1}{d}\sum_{i=1}^d \left(\alpha\vv_i + \beta\right)\v1\\
        &=\alpha\vv + \beta\v1 - \alpha\left(\frac{1}{d}\sum_{i=1}^d\vv_i\right)\v1 - \beta\v1\\
        &=\alpha\vv - \alpha\left(\frac{1}{d}\underbrace{\sum_{i=1}^d\vv_i}_{=0}\right)\v1\\
        &=\alpha\vv
    \end{split}
\end{equation}

Then, we scale the projected vector to be with a norm of $\sqrt{d}$ and get
\begin{equation}
    \text{LayerNorm}\left(\vx\right) = \sqrt{d}\frac{\alpha \vv}{||\alpha \vv||}
 \end{equation}
 We split to cases:\footnote{
    As implied from the original formulation of LayerNorm \cite{ba2016layer}, LayerNorm is undefined for $\alpha=0$.}
\begin{equation}
    \text{LayerNorm}\left(\vx\right) = \begin{cases}
        \sqrt{d} \vv  & \alpha > 0 \\
        -\sqrt{d} \vv  & \alpha < 0
      \end{cases}
\end{equation}
\newcommand{\RightAngle}[4][5pt]{%
  \draw[] ($#3!#1!#2$) -- ($ #3!2!($($#3!#1!#2$)!.5!($#3!#1!#4$)$) $) -- ($#3!#1!#4$) -- #3  -- cycle;
}
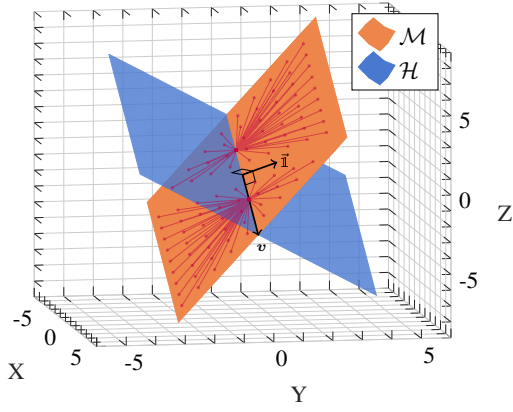
\begin{figure}
        \centering
\begin{tikzpicture}[trim axis left,trim axis right, clip, scale=0.80    ]

        \definecolor{darkslategray38}{RGB}{38,38,38}
        \definecolor{lightgray204}{RGB}{204,204,204}
        \definecolor{color0}{rgb}{0.282352941176471,0.470588235294118,0.815686274509804}
        \definecolor{color1}{rgb}{0.933333333333333,0.52156862745098,0.290196078431373}
        
        \begin{axis}[
            tick pos=both,
        axis line style={lightgray204},
        x grid style={lightgray204},
        xlabel=\textcolor{darkslategray38}{X},
        xmajorgrids,
        xmin=-6, xmax=6,
        xtick style={color=darkslategray38},
        xtick={-8,-7,-6,-5,-4,-3,-2,-1,0,1,2,3,4,5,6,7,8},
        xticklabels={,,,-5,,,,,0,,,,,5,,,},
        y grid style={lightgray204},
        ylabel=\textcolor{darkslategray38}{Y},
        ymajorgrids,
        ymin=-6, ymax=6,
        ytick style={color=darkslategray38},
        ytick={-8,-7,-6,-5,-4,-3,-2,-1,0,1,2,3,4,5,6,7,8},
        yticklabels={,,,-5,,,,,0,,,,,5,,,},
        z grid style={lightgray204},
        zlabel style={rotate=-90.0},
        zlabel=\textcolor{darkslategray38}{Z},
        zmajorgrids,
        zmin=-9, zmax=9,
        ztick style={color=darkslategray38},
        ztick={-8,-7,-6,-5,-4,-3,-2,-1,0,1,2,3,4,5,6,7,8},
        zticklabels={,,,-5,,,,,0,,,,,5,,,},
        zticklabel pos=right,
        view={80}{10},
        legend cell align={left},
        legend pos=north east,
        legend entries={$\mathcal{M}$, $\mathcal{H}$},
        ]

        \addplot3[surf,mesh/rows=2,mesh/cols=2,fill=color1,faceted color=color1, opacity=1, samples=2, samples y=2, domain=-3:3, y domain=-2.7:2.7] {-x+2*y};
        
        \addplot3[surf,mesh/rows=2,mesh/cols=2,fill=color0, faceted color=color0, draw=none,line width=0.00001pt, opacity=1, samples=2, samples y=2, domain=-3:3, y domain=-4:4] {-x-y+100};

        \addplot3[surf,mesh/rows=2,mesh/cols=2,fill=color0, faceted color=color0, draw=none,line width=0.00001pt, opacity=0.8, samples=2, samples y=2, domain=-3:3, y domain=-4:4] {-x-y};

        \addplot3[surf,mesh/rows=2,mesh/cols=2,fill=color1,faceted color=color1, opacity=1, samples=2, samples y=2, domain=-3:3, y domain=0:3] {-x+2*y};

        \coordinate (O) at (0,0,0);
        \coordinate (A) at (1,-2,1);
        \coordinate (B) at (1,1,1);
        \coordinate (C) at (1.25*2.449489742783179,0,1.25*-2.4494897427831774) ;

        \addplot3[->,draw=black,thick,no markers, domain=-3:3, y domain=-2:2] coordinates {
                (0,0,0) (1,1,1) 
        }       
        node[pos=1.2, opacity=1]{\scriptsize\color{black}{$\v1$}};

        \addplot3[->,draw=black,thick,no markers, domain=-3:3, y domain=-2:2] coordinates {
                (0,0,0) (1.25*2.449489742783179,0,1.25*-2.4494897427831774) 
        }       
        node[pos=1.2, opacity=1]{\scriptsize\color{black}{$\vv$}};

        \RightAngle{(A)}{(O)}{(B)};
        \RightAngle{(C)}{(O)}{(B)};

        \addplot3 [purple, opacity=0.5, mark=*, mark size=0.5, mark options={solid, draw=purple}, domain=-3:3, y domain=-2:2, x=x, y=y, z=z]
        file {figures/char_figure_data.data};
        \end{axis}
        
        \end{tikzpicture}
        
        \caption{LayerNorm maps the points of $\mathcal{M}$ to exactly two points in $\mathcal{H}$.}
        \label{fi:normalized_ap}
\end{figure}
To conclude, all vectors belonging to the 2D plane $\mathcal{M}$ are normalized to exactly two points, 
depending on their $\alpha$ component.

Since the subspaces $\mathcal{H}$ 
and $\left\{t\v1|t\in\mathbb{R}\right\}$ are direct sum of the whole space $\mathbb{R}^d$, each vector  $\vu\in\mathbb{R}^d$ has a unique representation as $\vu=\alpha \vv + \beta \v1$. That is, each vector $\vu\in\mathbb{R}^d$ belongs to some 2D plane $\mathcal{M}$, defined by $\v1$ and some vector $\vv\in\mathcal{H}$, and thus we can characterize the set of points that are being normalized to the same point.
\Cref{fi:normalized_ap} illustrates this behavior.

\section{Constructing the Projection}\label{Ap:projection}
The \emph{projection} of LayerNorm is a linear transformation, and thus, in this section, we show explicitly the construction of the projection matrix $\mP$, such that
\begin{equation}
    \mP\vx=\vx-\vmu
\end{equation}

Let $V=\mathbb{R}^d$,
$W=\left\{\alpha \v1 | \alpha \in \mathbb{R}\right\}$, and
$U=\left\{\vx | \vx \bot \v1=\left[1,1,...,1\right]\in\mathbb{R}^d\right\}$ be linear subspaces of $V$.

Let $B_U, B_W$ be the bases of $U$ and $W$ respectively:
\begin{align}
    B_U &= \left\{\vu_1,\vu_2,...,\vu_{d-1}\right\} \\
    B_W &= \left\{\v1\right\}
\end{align}
We can define a basis $C=B_U \cup B_W$ of $V$.

We also denote the standard basis $E$ of $V$: 
\begin{equation}
    E=\left\{\ve_1,\ve_2,...,\ve_d\right\}
\end{equation}

Since $U \cap W=\left\{0\right\}$, we have that $U \oplus W=V$ (direct sum).
Therefore, each $\vx\in V$ has a unique representation as $\vx=\vu+\vw$ where $\vu\in U$ and $\vw\in W$.

We can also write $\vx$ using $B_U, B_W$:
\begin{equation}
     \vx=\alpha \v1 + \sum^{d-1}_{i=1}\beta_i\vu_i
\end{equation}

Since we look for the projection of $\vx$ onto $U$ in the direction of $W$,
we want that
\begin{equation}
    \mP \vx=\sum^{d-1}_i\beta_i\vu_i
\end{equation}

To achieve this, we first, change the basis of $V$ from $E$ to $C$, remove the $\alpha \v1$ component, and then change back to the standard basis $E$.

Let $\mM^{C}_{E}$ be the change of basis matrix from basis C to the standard basis $E$:

\begin{align}
    \mM^{C}_{E} &= \begin{bmatrix}
        \text{|} & \text{|} &  & \text{|}  & \text{|} \\ 
        \vu_1 & \vu_2 & \hdots & \vu_{d-1}  & \v1 \\ 
        \text{|} & \text{|}&  & \text{|}  & \text{|} \\ 
         \end{bmatrix}
\end{align}
Therefore
\begin{equation}
    \mP = 
    \mM^{C}_{E}\mA\left(\mM^{C}_{E}
    \right)^{-1}
\end{equation}
Where
\begin{align}
    \mA &= \begin{bmatrix}
        \text{|} & \text{|} &  & \text{|}  & \text{|} \\ 
        \ve_1 & \ve_2 & \hdots & \ve_{d-1}  & \bm{0} \\ 
        \text{|} & \text{|}&  & \text{|}  & \text{|} \\ 
         \end{bmatrix}
\end{align}

To get an explicit $\mP\in\mathbb{R}^{d\times d}$, we instantiate the basis $B_U$:
\begin{align}
    B_U &= \left\{
         \begin{bmatrix}
            1-d \\
            1 \\ 
            \vdots \\
            1 \\ 
            1 
             \end{bmatrix},
        \begin{bmatrix}
            1 \\
            1-d\\ 
            \vdots \\
            1 \\ 
            1 
        \end{bmatrix}, ...,
        \begin{bmatrix}
            1 \\
            1 \\ 
            \vdots \\
            1-d \\
            1
             \end{bmatrix}
         \right\}
\end{align}
Therefore
\begin{align}
    \mM^{C}_{E} &= \begin{bmatrix}
        1-d & 1 & \cdots & 1 & 1 \\ 
        1 & 1-d & \cdots & 1 & 1 \\ 
        \vdots & \vdots & \ddots & \vdots  & \vdots \\ 
        1 & 1 & \cdots & 1-d & 1 \\
        1 & 1 & \cdots & 1 & 1
         \end{bmatrix}
\end{align}
\begin{align}
    \left(\mM^{C}_{E}\right)^{-1} &= \frac{1}{d}\begin{bmatrix}
        -1 &  &  &  & 1 \\ 
           & -1 &  &  & 1 \\ 
           &  & \ddots &  & \vdots \\ 
           &  &  & -1 & 1 \\ 
        1 & 1 & \cdots & 1 & 1 \\ 
         \end{bmatrix}
\end{align}
And we get
\begin{align}
    \mP &= \frac{1}{d}\begin{bmatrix}
        d-1 & -1 & \cdots& -1  \\ 
        -1 & d-1 & \cdots & -1  \\
        \vdots & \vdots & \ddots & \vdots \\
        -1 & -1 & \cdots & d-1
    \end{bmatrix}
\end{align}

\section{Unselectable Keys}\label{Ap:unargmaxable-keys}
\Cref{tab:unargmaxable_ln} shows the fraction of ``unselectable'' keys in a language model with LayerNorm before and after the application of LayerNorm.

\begin{table}[ht!]
    \centering
        \begin{tabu}{lccccc}
        \toprule
         & $L_1$ & $L_2$ & $L_3$ & $L_4$ \\
        \midrule
        Before LayerNorm & 44.8 & 28.5 & 22.3 & 26.1 \\
        After\hphantom{tf} LayerNorm & \textbf{0} & \textbf{0} & \textbf{0} & \textbf{0} \\
    \bottomrule
    \end{tabu}
    \caption{The fraction of ``unselectable'' key vectors before and after the LayerNorm followed by the attention mechanism of each layer of a language model. LayerNorm solves the ``unselectable'' keys problem using the \emph{scaling} property.}
    \label{tab:unargmaxable_ln}
\end{table}

\begin{table*}[ht!]
    \centering
        \begin{tabu}{l}
        \toprule
        whether you like rap music or loathe it, you can't deny either the tragic loss of two young men \\in the prime of their \textbf{talent} or the \textbf{power} of this movie. \\
        \\
        it is great summer \textbf{fun} to watch arnold and his buddy gerald bounce off a quirky cast of \\characters. \\
        \\
        the lion king was a roaring \textbf{success} when it was released eight years ago, but on imax it seems\\ better, not just bigger. \\
        \\
        it provides the grand, intelligent \textbf{entertainment} of a superior cast playing smart people amid \\a compelling plot. \\
        \\
        some of their jokes work, but most fail miserably and in the end, pumpkin is \textbf{far} more \\\textbf{offensive} than it is funny. 
        \\
        \bottomrule
    \end{tabu}
    \caption{Examples from the validation set of Stanford TreeBank \cite{socher-etal-2013-recursive}. Any token that is ``unselectable' in at least one of the layers of the language model (\Cref{Se:Unbounded}) is marked in \textbf{bold}.}
    \label{tab:sst2-decoder-unselectable}
\end{table*}

\begin{table*}[ht!]
    \centering
        \begin{tabu}{l}
        \toprule
        it's hard to \textbf{like} a film about a guy who is utterly unlikeable, and shiner, starring michael caine \\as an aging british boxing promoter desperate for a taste of fame and fortune, is certainly that.\\
        \\
        you'll gasp appalled and \textbf{laugh} outraged and possibly, watching the spectacle of a promising \\young lad treading desperately in a nasty sea, shed an errant tear. \\
        \\
        this is \textbf{wild} surreal \textbf{stuff}, but brilliant and the camera just kind of sits there and lets you look \\at this and its like you're going from one room to the next and none of them have any relation \\to the other. \\
        \\
        it’s a much more \textbf{emotional} journey than what shyamalan has given us in his past two movies, \\and gibson, stepping in for bruce willis, is the perfect actor to take us on the trip. \\
        \\
        \textbf{although} german cooking does not come readily to \textbf{mind} when considering the world's best \\cuisine, mostly martha could make deutchland a popular destination for hungry tourists. \\
        \bottomrule
    \end{tabu}
    \caption{Examples from the validation set of Stanford TreeBank \cite{socher-etal-2013-recursive}. Each \textbf{bold} token is ''unselectable' in at least one of the layers of a Transformer encoder without LayerNorm, trained on the sentiment analysis task. These examples show that important tokens that may be necessary for the task are ''unselectable'', which may affect the encoder's ability to learn the task correctly.}
    \label{tab:sst2-encoder-unselectable}
\end{table*}
To illustrate the impact of ''unselectable'' tokens, we give some examples from the validation set of Stanford TreeBank \cite{socher-etal-2013-recursive}, which is used as a benchmark for the sentiment analysis task. Our results show that important tokens may be ''unselectable''. We highlighted in bold any token that is ''unselectable' in at least one of the layers.
\Cref{tab:sst2-decoder-unselectable} shows the results of running a language model without LayerNorm (\Cref{Se:Unbounded}) on the validation set.
We also trained a 4-layer Transformer encoder without LayerNorms (based on BERT architecture \cite{devlin2018bert}) on the sentiment analysis task.
\Cref{tab:sst2-encoder-unselectable} shows the results of running this model on the validation set.

\section{Experimental Setup}\label{Ap:Setup}
In this section, we detail the setup of the experiments shown in \Cref{Se:Results}.

\subsection{Majority}
In the experiments, we used a learning rate of $0.001$ with a linear scheduler, a hidden size of $d = 8$ (total of $584$ learnable parameters), a batch size of $6000$, a sequence length of $50$, $20$ different classes, and the Adam optimizer.
We trained the models for 1000 epochs consisting of 17K steps.

The ``majority'' dataset contains 80K training examples and 20K test examples. Each example is a sequence of length $50$ consisting of tokens belonging to one of $20$ different classes.

\subsection{Language Modeling}
We trained a language model with GPT2 architecture \cite{radford2019language} using the Huggingface library, on the Huggingface processed
Wikipedia dataset (20220301.en split, licenses CC BY-SA and CC BY-SA) \cite{wikidump}, and tested it on SQuAD \cite{squad} (license CC BY-SA 4.0). We used these datasets only to demonstrate the ``unselectable'' keys problem, and thus we did not violate any of their license conditions.
We used the same hyperparameters as \citet{radford2019language}, except that we used a hidden size of $8$, $4$ hidden layers, a learning rate of $5\text{e-}5$, and a window size of $1024$ tokens, resulting in a model with 414K learnable parameters.
We train the model on the Wikipedia dataset (6.5M examples) for 50K steps and report our findings on 1000 randomly selected examples from the validation set of SQuAD. 

\subsection{Sentiment Analysis Task}
We trained a Transformer encoder with BERT architecture \cite{devlin2018bert} without LayerNorm layers with 50K steps on the Stanford TreeBank dataset \cite{socher-etal-2013-recursive} (\Cref{tab:sst2-encoder-unselectable}).
We used the same hyperparameters as \citet{devlin2018bert}, except that we used a hidden size of $8$, $4$ hidden layers, and a learning rate of $5\text{e-}5$, resulting in a model with 446K learnable parameters.

\end{document}